\documentclass{article}

\usepackage{graphicx}
\usepackage{paper}

\usepackage[accepted]{icml2018}

\setcitestyle{authoryear,round,citesep={;},aysep={,},yysep={;}}

\icmltitlerunning{Reinforcing Adversarial Robustness using Model Confidence
  Induced by Adversarial Training}

\newcommand{\HCNN}{{\tt HCNN}}
\newcommand{\MCN}{{\tt MCN}}
\newcommand{\NCN}{{\tt NCN}}
\newcommand{\OracleHCNN}{{\rm OracleHCNN}}

\begin{document}
\twocolumn[                     
\icmltitle{
  Reinforcing Adversarial Robustness using Model Confidence \\
  Induced by Adversarial Training
}
\icmlsetsymbol{equal}{*}
\begin{icmlauthorlist}
  \icmlauthor{Xi Wu}{equal,goo}
  \icmlauthor{Uyeong Jang}{equal,uwm}
  \icmlauthor{Jiefeng Chen}{uwm}
  \icmlauthor{Lingjiao Chen}{uwm}
  \icmlauthor{Somesh Jha}{uwm}
\end{icmlauthorlist}

\icmlaffiliation{uwm}{University of Wisconsin-Madison}
\icmlaffiliation{goo}{Google}

\icmlcorrespondingauthor{Xi Wu}{xiwu@cs.wisc.edu}
\vskip 0.3in
]                               


\printAffiliationsAndNotice{\icmlEqualContribution} 

\begin{abstract}
  In this paper we study leveraging \emph{confidence information} induced by adversarial training
  to reinforce adversarial robustness of a given adversarially trained model.
  A natural measure of confidence is $\|F(\bfx)\|_\infty$
  (i.e. how confident $F$ is about its prediction?).
  We start by analyzing an adversarial training formulation proposed by Madry et al..
  We demonstrate that, under a variety of instantiations, an only somewhat good solution to their objective
  induces confidence to be a discriminator, which can distinguish between
  right and wrong model predictions in a neighborhood of a point sampled from the underlying distribution.
  Based on this, we propose Highly Confident Near Neighbor ($\HCNN$),
  a framework that combines confidence information and nearest neighbor search,
  to reinforce adversarial robustness of a base model. We give algorithms in this framework
  and perform a detailed empirical study. We report encouraging experimental results
  that support our analysis, and also discuss problems we observed with existing adversarial training.
\end{abstract}


\section{Introduction}
\label{sec:intro}
In the adversarial-perturbation problem for neural networks, an adversary starts with a neural network
model $F$ and a point $\bfx$ that $F$ classifies correctly (we assume that $F$ ends with a softmax layer,
which is common in the literature), and crafts a small perturbation to produce another point $\bfx'$
that $F$ classifies \emph{incorrectly}. \cite{SZSBEGF13} first noticed the vulnerability of
existing neural networks to adversarial perturbations, which is somewhat surprising given their
great generalization capability. Since then, a line of research
(e.g.,~\cite{GSS14, PMWJS16, MMKI17, MMSTV17}) has been devoted to hardening neural networks
against adversarial perturbation. While modest progress has been made, until now there is still a large gap
in successfully defending against more advanced attacks, such as the attack by~\cite{CW17}.

In this paper we study leveraging \emph{confidence information} induced by adversarial training,
to reinforce adversarial robustness of a given adversarially trained model.
A natural measure of confidence is $\|F(\bfx)\|_\infty$
(i.e. how confident $F$ is about its prediction?).
Our motivation comes from a thought experiment based on the manifold assumption~\cite{ZG09} made in
unsupervised and semi-supervised learning, which states that
\emph{natural data points lie on (or near to) separate low dimensional manifolds for different classes.}
If one believes that a good deep neural network model can approximate well the natural manifolds,
then ideally it should have the property that it can \emph{confidently} distinguish points
from natural manifolds, while not claiming confidence for points that are far away from
the natural manifolds. Since natural manifolds are assumed to be separate,
it follows that for a good model,
\emph{confident model predictions for different classes must be separate from each other.}

Taking this perspective, we formalize a probabilistic property on the separation of confident
model predictions for different classes, and use it to examine existing adversarial training formulations.
Somewhat surprisingly, we find that a natural min-max formulation proposed
by~\cite{MMSTV17}, under a variety of instantiations, encourages training a model
that satisfies our property well. At a high level, the reason for this is that a good solution to the
inner maximization of the min-max formulation encourages that there is \emph{no confident wrong}
prediction in a neighborhood of a point sampled from the underlying distribution.

Our key observation is that, even if a model is only \emph{a somewhat good} solution under
Madry et al.'s adversarial training formulation, it \emph{induces} confidence to be
a discriminator which can distinguish between right and wrong model predictions.
That is, a model prediction of higher confidence is more likely to be correct.
As a result, even though a somewhat good adversarially trained model
may still have many wrong predictions \emph{with low confidence},
one can use confidence to further reinforce its adversarial robustness.
For example, one can now try to protect the model, from making wrong predictions,
by \emph{rejecting} a point if model confidence at the point is below certain threshold.

We take a step further to study correcting model predictions at adversarial points,
leveraging that model predictions there tend to have lower confidence.
We propose Highly Confident Near Neighbor ($\HCNN$),
a framework which combines confidence information and nearest neighbor search,
to embed a possibly low-confidence point back to high-confidence regions.
Specifically, let $\|\cdot\|$ be a norm, and suppose that we want to defend against $\|\cdot\|$-attacks
with radius $\eta$. Let $F$ be a model with good separation property so that confidence
can distinguish right and wrong model predictions in a $\delta$-neighborhood of a point
sampled from the underlying distribution. Let $|\cdot|$ be a possibly different norm,
and $\xi > 0, \lambda \ge 0$ be two real numbers.
The $\xi$-Highly Confident Near Neighbor, or simply $\HCNN_\xi$, computes
$\argmax_{\bfz \in N(\bfx, \xi)}(\|F(\bfz)\|_\infty - \lambda|\bfz - \bfx|)$,
where $N(\bfx, \xi)$ is the $\xi$-neighborhood around $\bfx$ w.r.t. $\|\cdot\|$.
That is, in a $\xi$-ball centered at $\bfx$, we encourage finding a point which simultaneously has
a high-confidence prediction, and is near to $\bfx$. Therefore, one can hope to correct adversarial points
for $\eta \le \delta - \xi$. It turns out in our experiments that there are nontrivial settings of
$\delta, \eta, \xi$ under which we demonstrate a significant improvement of model robustness over $F$.

We perform a detailed empirical study over CIFAR10 for $\ell_\infty$ attacks.
We reuse the robust ResNet model trained by Madry et al. as \emph{base model},
and use $\HCNN_\xi$. We modify state-of-the-art $\ell_\infty$ attacks, such as the CW attack~\cite{CW17},
and the PGD attack~\cite{MMSTV17}, to \emph{exploit confidence information in order to break our method by
  generating high-confidence attacks}. We first empirically validate that Madry et al.'s model is better,
in view of our probabilistic separation property, than models trained without a robustness objective.
We then evaluate using confidence to reject adversarial examples, and finally end-to-end defense results.
Our results are both encouraging and discouraging: for small radius, we find that confidence is indeed
a good discriminator to distinguish right and wrong predictions, and it does improve adversarial robustness.
For large radius that approaches $.03$, however, we find that Madry et al.'s base model loses control there.
We further argue that modifications are needed for Madry et al.'s adversarial training in order to
defend against more effectively $\ell_\infty$-attacks of large radius.

In summary this paper makes the following contributions:
\begin{itemize}
\item We propose a probabilistic property of separation between confident predictions for different classes.
  We prove that a natural min-max adversarial training encourages a model that satisfies this property well.
  As a result, confidence can be used as a discriminator to distinguish between right and wrong predictions,
  even when a model is only a somewhat good solution to the training objective.

\item We propose using confidence induced by adversarial training to further reinforce adversarial
  robustness of a given adversarially trained model. Specifically, we propose rejecting adversarial points
  using confidence, and further correcting adversarial points by combining confidence with nearest neighbor search.

\item We perform a detailed empirical study to validate our proposal.
  We report encouraging results that support our analysis,
  and discuss problems we observed with existing adversarial training.
\end{itemize}

The rest of the paper is organized as follows: We start with some preliminaries in Section~\ref{sec:preliminaries}.
Then Section~\ref{sec:goodness-property} proposes the probabilistic separation property and we use it to
analyze Madry et al.'s adversarial training formulation. We then present embedding objectives and algorithms
for handling low-confidence points, and end-to-end instantiations, in Section~\ref{sec:defense}.
Section~\ref{sec:experiments} performs a detailed empirical study of our method.
We discuss important prior work in Section~\ref{sec:related} and conclude in Section~\ref{sec:conclusion}.


\section{Preliminaries}
\label{sec:preliminaries}
As in existing work, such as~\cite{CW17, PMWJS16}, we define $F$ to be a neural network
after the softmax layer. With this notation, the final classification is then
$C_F(\bfx) = \argmax_i F(\bfx)_i$ where $F(\bfx)_i$ gives the confidence of the network
in classifying $\bfx$ for class $i$. We use $Z(\bfx)$ to denote part of $F$ except the softmax layer.
That is, $Z(\bfx)$ computes the \emph{logits} to be fed into the softmax function.
Let $\calC$ denote the class of all labels. A network is typically parameterized by parameters $\theta$,
and in this case we denote by $F_\theta$ to indicate that $F$ is parameterized by $\theta$.
Finally, let $p \in [0, 1]$ be a parameter, and $l \in \calC$.
A point $\bfx$ is $p$-confident for label $l$ if $F(\bfx)_l \ge p$.
We consider the following adversarial model which is implicit in several previous work
(e.g.~\cite{MMSTV17}):
\begin{definition}[\textbf{Adversarial Model}]
  \label{def:nap}
  Let $\calD$ be a data generating distribution over $\Real^d$,
  $F$ be a model, and $\calS \subseteq \Real^d$ be a set of perturbations.
  We consider the following game between an adversary and a defender:\\
  \textbf{Adversary}: the adversary draws $\bfx \sim \calD$,   produces $\bfx' = \bfx + \Delta$ for some
  $\Delta \in \calS$, and sends $\bfx'$ to the defender.\\
  \textbf{Defender}: the defender outputs a label $l \in \calC$ for $\bfx'$.\\
  The defense succeeds if $C_F(\bfx) = C_F(\bfx')$.
\end{definition}
An important point regarding this adversarial model is that it only considers points on $\calD$ or nearby
(as specified by the allowable perturbations $\calS$), instead of \emph{the entire domain}.
While seemingly one should consider every point in the space,
this definition better reflects the intuition behind adversarial perturbation problem:
that is, for ``natural images,'' the classification should be consistent
with respect to a set of ``small perturbations.''


\section{Probabilistic Separation Property and Madry et al.'s Formulation}
\label{sec:goodness-property}
This section develops the following:
\begin{itemize}
\item In Section~\ref{sec:goodness-property:formalization}, we propose a goodness property of models
  which states that \emph{confident regions of a good model should be well separated.}
  We further formalize this as a probabilistic separation property.
\item In Section~\ref{sec:goodness-property:madry}, we examine existing robust training formulations
  and demonstrate that a natural min-max formulation of~\cite{MMSTV17} encourages very good separation
  of $p$-confident points of different classes for large $p$, but has a weak control of low-confidence
  wrong predictions due to estimation error.
\end{itemize}

\subsection{Probabilistic Separation Property}
\label{sec:goodness-property:formalization}
The manifold assumption in unsupervised and semi-supervised learning states that
\emph{natural data points lie on (or near to) separate low dimensional manifolds for different classes.}
Under this assumption what would an ideal model look like?
Clearly, we would expect that an ideal model can \emph{confidently} classify points from the manifolds,
while \emph{not} claiming confidence for points that are far away from those manifold.
Therefore, since by assumption low dimensional manifolds are separate from each other,
we would expect that confident correct predictions of different classes are separate from each other.
Therefore, we propose the following goodness property
\begin{center}
  \emph{Confident predictions of different classes should be well separated.}
\end{center}
We formalize this as the following property:
\begin{definition}[$(p, q, \delta)$-separation]
  \label{def:p-q-delta-separation}
  Let $\calD$ be a data generating distribution,
  $p, q \in [0, 1]$, $\delta \ge 0$, and $d(\cdot, \cdot)$ be a distance metric.
  Let $\calB$ be the event $\{ \exists y' \neq y, \bfx' \in N(\bfx, \delta), F_\theta(\bfx')_{y'} \ge p \}$.
  $F$ is said to have $(p, q, \delta)$-separation if $\Pr_{(\bfx,y) \sim \calD}\big[\calB\big] \le q,$
  where $N(\bfx, \delta) = \{ \bfx'\ |\ d(\bfx, \bfx') \le \delta \}$.
\end{definition}
This definition says that for a point $(\bfx, y)$ sampled from the data generating distribution,
in its neighborhood there should \emph{not} be \emph{confident} wrong predictions
(though wrong model predictions with low confidence may still exist).
Thus, if a model satisfies this definition well, then confident predictions of different classes
must be well separated. Finally, this definition allows certain points to be ``bad.''

\subsection{An Analysis of the Min-Max Formulation of~\cite{MMSTV17}}
\label{sec:goodness-property:madry}
This section gives an analysis of the min-max formulation proposed by~\cite{MMSTV17}.
Essentially, our analysis shows that a variety of its instantiations encourage training a model
that satisfies Definition~\ref{def:p-q-delta-separation}.
More specifically, we prove two things:
(1) the formulation encourages training a model that satisfies Definition~\ref{def:p-q-delta-separation}, and
(2) complementing this, we show that the formulation may have a weak control of points with low-confidence
but wrong predictions, due to estimation error with finite samples.
To start with, Madry et al. proposes the following formulation:
\begin{align}
  \label{def:madry-objective}
  \begin{split}
    &\minimize \rho(\theta), \\
    &\text{where } \rho(\theta) 
    = \Exp_{(\bfx, y) \sim \calD}\left[\max_{\Delta \in \calS} L(\theta, \bfx + \Delta, y)\right],
  \end{split}
\end{align}
where $\calD$ is the data generating distribution, $\calS$ is set of allowed perturbations
(e.g., $\calS = \{ \Delta\ |\ \|\Delta\|_\infty \le \delta\}$),
and $L(\theta, \bfx, y)$ is the loss of $\theta$ on $(\bfx, y)$.
For the rest of the discussion, we denote
$\kappa(\theta, \bfx, y) = \max_{\Delta \in \calS}L(\theta, \bfx + \Delta, y)$,

\noindent\textbf{Analyzing a Family of Loss Functions.}
Let us consider the following family of loss functions,
\begin{align*}
  \calL = \big\{ & L(\theta, \bfx, y) \text{ where $L(\theta, \bfx, y)$ is monotonically } \\
                                & \text{decreasing in $F_\theta(\bfx)_y$}. \big\}
\end{align*}
In other words, as we have higher confidence about the correct prediction,
we have smaller loss. For $L \in \calL$, there is a monotonically decreasing
\emph{loss-lower-bound function} $\tau_L: [0, 1] \mapsto \Real^{\ge 0}$,
so that if $F_\theta(\bfx)_y \le q$, then $L(\theta, \bfx, y) \ge \tau_L(q)$.

\begin{example}[\textbf{Cross Entropy Loss}]
  Let $L(\theta, \bfx, y) = \Ent(\mathbf{1}_y, F_\theta(\bfx)) = -\log F_\theta(\bfx)_y$,
  be the cross entropy between $\mathbf{1}_y$ and $F_\theta(\bfx)$,
  where $F_\theta$ is the model instantiated with parameters $\theta$,
  and $\mathbf{1}_y$ is the indicator vector for label $y$.
  Then we can define $\tau_L$ as $\tau_L(\alpha) = -\log \alpha$.
\end{example}

With these notations we have the following proposition,

\begin{proposition}
  \label{proposition:confident-separation}
  Let $L \in \calL$, $\tau_L$ be a loss-lower-bound function,
  and $\calB$ be the event $\{(\exists y' \neq y, \bfx' \in \bfx + \calS)\ F_\theta(\bfx')_{y'} \ge p\}$.
  If $\rho(\theta) \le \varepsilon$, then
  $\Pr_{(\bfx, y) \sim \calD}\big[ \calB\big] \le \frac{\varepsilon}{\tau_L(1-p)}.$
  That is, the probability of an $\bfx' \in \bfx + \calS$ that is $p$-confident on a wrong label
  decreases as $p \rightarrow 1$.
\end{proposition}
\begin{proof}
  Let $\alpha = 1 - p$. If $\calB$ happens, then $F_\theta(\bfx')_y \le \alpha$ for some $\bfx' \in \bfx + \calS$,
  and $\kappa(\theta, \bfx, y)  \ge \tau_L(\alpha)$, therefore\footnote{
    Let $X$ be a nonnegative random variable and $a > 0$, Markov's inequality says that
    $\Pr[X \ge a] \le \Exp[X]/a$.},
  \begin{align*}
    \Pr_{(\bfx,y) \sim \calD}[\calB]
    \le &\Pr_{(\bfx, y) \sim \calD}\big[\kappa(\theta, \bfx, y) \ge \tau_L(\alpha)\big] \\
    \le &\frac{\Exp[\kappa(\theta, \bfx, y)]}{\tau_L(\alpha)} \qquad\qquad \text{
          (Markov's inequality)}\\
    \le &\frac{\varepsilon}{\tau_L(\alpha)}.
  \end{align*}
  The proof is complete.
\end{proof}
This immediately generates the following corollary,
\begin{corollary}
  Let $\calS$ be a region defined as $\{ \Delta\ |\ d(\Delta, {\bf 0}) \le \delta \}$.
  If $\rho(\theta) \le \varepsilon$, then the model $F_\theta$ is
  $(p, \frac{\varepsilon}{\tau_L(1-p)}, \delta)$-separated.
\end{corollary}
This indicates that even if a model is only a somewhat good solution to (\ref{def:madry-objective}),
meaning $\rho(\theta) \le \varepsilon$ for only a somewhat small $\varepsilon$,
then $p$-confident points will be well separated
(in the sense of Definition~\ref{def:p-q-delta-separation}) as soon as $p$ increases.

The above proposition considers a situation where we have a confident but wrong point
(with confidence $p$). What about points that have wrong but low-confidence predictions?
For example, the confidence on the wrong label is only $\frac{1}{2} + \nu$ for some small $\nu$?
Note that by setting $p = 1/2$ we derive immediately a bound $\varepsilon/\tau_L(1/2)$
(which can be much weaker than $\varepsilon/\tau_L(1-p)$ for large $p$).
We note that this bound is tight without further assumptions:

\begin{proposition}
  \label{proposition:unconfident-separation}
  Let $F_\theta$ be a neural network parameterized by $\theta$,
  and $C_{F_\theta}$ be the corresponding classification network. If $\rho(\theta) \le \varepsilon$, then
  $\Pr_{(\bfx, y) \sim \calD}\big[(\exists y' \neq y, \bfx' \in \bfx + \calS),\ C_{F_\theta}(\bfx')
    = y'\big] \le \frac{\varepsilon}{\tau_L(1/2)}.$ The bound is tight.
\end{proposition}
\begin{proof}
  Let $\calB$ be the event $\{(\exists y' \neq y, \bfx' \in \bfx + \calS)\ C_{F_\theta}(\bfx') = y'\}$.
  If $\calB$ happens then $F_\theta(\bfx')_y \le \frac{1}{2}$ (otherwise $\bfx'$ will be classified as $y$),
  and so $\kappa(\theta, \bfx, y) \ge \tau_L(1/2)$. On the other hand,
  if $\calB$ does not happen, then we can lower bound $\kappa(\theta, \bfx, y)$ by $0$.
  Therefore $\varepsilon \ge \Exp[\kappa(\theta, \bfx, y)] \ge \Pr[\neg\calB]
  \cdot 0 + \Pr[\calB] \cdot \tau_L(1/2) = \Pr[\calB] \cdot \tau_L(1/2)$.
  Tightness follows as we can force equality for each of the inequalities. The proof is complete.
\end{proof}

\noindent\textbf{Instantiations.}
The above two propositions can be specialized to many different forms of $L$. Basically, any loss function
that encourages high confidence on the correct label, when plugged into (\ref{def:madry-objective}),
encourages that there is no wrong prediction with high confidence. We list a few of them:
\begin{itemize}
\item \textbf{\em Cross Entropy Loss}. If we use $L(\theta, \bfx, y) = -\log F(\bfx)_y$, we immediately
  get separation guarantee $-\frac{\rho(\theta)}{\log(1-p)}$
  in Proposition~\ref{proposition:confident-separation}.
\item \textbf{\em Squared Loss}. It is also natural to consider $L(\theta, \bfx, y) = (1 - F_\theta(\bfx)_y)^2$.
  This then gives separation guarantee $\rho(\theta)/p^2$. Note that this guarantee is weaker than
  the one above with cross entropy: As $p \rightarrow 1$, the probability of bad events happening
  only converges to $\rho(\theta)$, instead of $0$.
\item \textbf{\em Entropy Regularization}. In~\cite{MMKI17} the authors proposed to use the entropy
  of $F(\bfx)$, $\Ent(F(\bfx))$, as a regularizer. We note that such regularization is compatible with
  our argument here. Essentially as one increases $F(\bfx)_y$, the entropy decreases.
\item \textbf{\em Loss functions in Carlini-Wagner Attacks.}
  It is also not hard to check that any objective function used for targeted attacks in~\cite{CW17}
  (see Objective Function on pp. 6 of their paper) can be adapted as loss function $L$ in the
  min-max paradigm to draw similar conclusions.
\end{itemize}

Contrasting Proposition~\ref{proposition:confident-separation}
and~\ref{proposition:unconfident-separation}, we note the following:
\begin{itemize}
\item We note that in reality we only have a finite-sample approximation for (\ref{def:madry-objective}).
  Therefore, due to estimation error of solving the objective,
  even though one can hope for a good separation between $p$-confident points from different classes
  for large $p$ (by Proposition~\ref{proposition:confident-separation}),
  there may still be many \emph{low-confidence wrong} predictions
  (by Proposition~\ref{proposition:unconfident-separation}).

\item Our analysis indicates that even if the model is only a somewhat good solution to Madry et al.'s
  formulation, confidence gives additional information, and is provably a \emph{discriminator} which
  can distinguish between right and wrong predictions in a neighborhood of a point.
  Therefore in this situation, one can leverage confidence to \emph{reinforce robustness of the model}.
  For example, we can now try to apply confidence to protect a model with good separation property,
  from making wrong predictions, by \emph{rejecting adversarial points.}
  Specifically, with an appropriately chosen threshold $p_0$, we can modify the model so that
  we output $\bot$ if $\|F(\bfx)\|_\infty < p_0$, otherwise we output $C_F(\bfx)$.
  In this way, we hope that we can correctly predict on natural points
  while rejecting adversarial points. In the next section, we investigate further combining
  confidence and search to correct adversarial examples.
\end{itemize}


\section{From Confidence to Defenses}
\label{sec:defense}
Our analysis from the previous section shows that there can be good models with well-separated confident
regions,yet there may be a weak control of points that have wrong predictions but low model confidence.
However, since now right and wrong predictions are distinguishable by confidence,
one can try to correct an adversarial point by embed it back to high-confidence regions.
In the following we establish embedding and end-to-end instantiations.

\noindent\textbf{Embedding Objectives.}
We propose a framework, Highly Confident Near Neighbor ($\HCNN$), for embedding:
\begin{definition}[\textbf{Highly Confident Near Neighbor}]
  Let $F$ be a model, $\|\cdot\|, |\cdot|$ be two possibly different norms,
  and $\xi > 0, \lambda \ge 0$ be two real numbers.
  The \emph{$\xi$-Highly Confident Near Neighbor} objective (or simply $\HCNN_\xi$) computes:
  \begin{align}
    \label{eq:mcnn-def}
    \HCNN_\xi(F, \bfx) \equiv \argmax_{\bfz \in N(\bfx, \xi)}(\|F(\bfz)\|_\infty - \lambda|\bfz - \bfx|).
  \end{align}
  where $N(\bfx, \xi)$ is the $\xi$-ball around $\bfx$ with respect to $\|\cdot\|$.
\end{definition}
In short, in a $\xi$-ball centered at $\bfx$, this objective tries to encourage finding a point which
simultaneously has a high-confidence prediction, and is near to $\bfx$. We have two immediate variants.
First, by setting $\lambda = 0$, we have $\xi$-Most Confident Neighbor ($\MCN_\xi$):
\begin{definition}[\textbf{Most Confident Neighbor}]
  Let $F$ be a model, $\|\cdot\|$ be a norm, and $\xi > 0$ be a real number.
  The \emph{$\xi$-Most Confident Neighbor} objective (or simply $\MCN_\xi$) computes:
  \begin{align}
    \label{eq:mcn-def}
    \MCN_\xi(F, \bfx) \equiv \argmax_{\bfz \in N(\bfx, \xi)}\|F(\bfz)\|_\infty.
  \end{align}
  where $N(\bfx, \xi)$ is the $\xi$-ball around $\bfx$ with respect to $\|\cdot\|$.
\end{definition}
Second, by putting a threshold on $\|F(\bfz)\|_\infty$ and minimizing $|\bfz - \bfx|$ alone,
we have $p$-Nearest Confident Neighbor:
\begin{definition}[\textbf{Nearest Confident Neighbor}]
  Let $F$ be a model, $\|\cdot\|, |\cdot|$ be two possibly different norms, and $p \in (0, 1)$ be a real number.
  The $p$-\emph{Nearest Confident Neighbor} objective (or simply $\NCN_p$) computes:
  \begin{align}
    \label{eq:ncn-def}
    \begin{split}
      \NCN_p(F, \bfx) \equiv &\argmin_{\bfz \in N(\bfx, \xi)}|\bfz - \bfx|, \\
      &\text{ subject to } \|F(\bfz)\|_\infty \ge p.
    \end{split}
  \end{align}
  where $N(\bfx, \xi)$ is the $\xi$-ball around $\bfx$ with respect to $\|\cdot\|$.
\end{definition}


\begin{algorithm}[htbp]
  \caption{Solving $\HCNN_\xi$ by solving for each label.}
  \label{alg:oracle-hcnn}
  \begin{algorithmic}[1]
    \Input{
      $\bfx$ a feature vector,
      $\xi > 0$ a real parameter,
      $\lambda \ge 0$ a real parameter,
      a base model $F$,
      any gradient-based optimization algorithm $\cal O$ to solve
      the constrained optimization problem defined in (\ref{eq:mcnn-label}).}
    \Function{$\OracleHCNN$}{$\bfx, \xi, F$}
      \For{$l \in \calC$}
         \Let{$\bfz^{(l)}$}{${\calO}(\bfx, F, l)$}
      \EndFor
      \State \Return $\bfz^{(l^*)}$ where
      $$l^* = \argmax_{l \in \calC} \left(F(\bfz^{(l)})_l - \lambda|\bfz^{(l)} - \bfx|\right).$$
    \EndFunction
  \end{algorithmic}
\end{algorithm}

\noindent\textbf{Algorithms.}
For the rest of the paper we focus on solving $\HCNN_{\xi}$ (and hence $\MCN_\xi$ as well).
To start with, we note that the optimization for solving $\HCNN_\xi$ is nothing
but for each label $l \in \calC$ we try to find
\begin{align}
  \label{eq:mcnn-label}
  \bfz^{(l)} = \argmax_{\bfz \in N(\bfx, \xi)} (F(\bfz)_l - \lambda|\bfz - \bfx|)
\end{align}
and then compute $\bfz^{(l^*)}$ for $l^* = \argmax_{l} F(\bfz^{(l)})_l - \lambda|\bfz^{(l)} - \bfx|$.
(\ref{eq:mcnn-label}) can be solved using any preferred gradient-based optimization
(e.g., projected gradient descent~\cite{NW06}). This gives Algorithm~\ref{alg:oracle-hcnn}.

\noindent\textbf{\em Attacks as Defenses.}
Note that solving (\ref{eq:mcnn-label}) is similar to \emph{a targeted adversarial attack}
which modifies $\bfx$ to increase the confidence on target label $t$,
while maintaining small distance to $\bfx$. Thus {\em an adversarial attack},
such as those proposed in~\cite{CW17}, can be used as $\calO$.

\noindent\textbf{\em Entropy Approximations.}
Algorithm~\ref{alg:oracle-hcnn} may be inefficient. To mitigate this issue, one can replace
$\|F(\cdot)\|_\infty$ by entropy functions, which are then differentiable and so
we can directly apply gradient descent. We have the following:

\noindent\emph{Shannon entropy approximation}. Approximate $\|F(\bfz)\|_\infty$ by
  $\Ent(F(\bfz)) = -\sum_{l \in \calC}F(\bfz)_l\log F(\bfz)_l.$

\noindent\emph{$\alpha$-R\'{e}nyi entropy approximation}. For $\alpha \ge 0, \alpha \neq 1$,
approximate $\|F(\bfz)\|_\infty$ by $\Ent_{\alpha}(F(\bfz)) = \frac{\alpha}{1-\alpha}\log(\|F(\bfz)\|_\alpha),$
where $\|F(\bfz)\|_\alpha$ is the $\alpha$-norm of $F(\bfz)$.

Essentially, these approximations share some of the most important optima as $\|F(\bfz)\|_\infty$:
for example if the probability vector $F(\bfz)$ concentrates on one coordinate.

\noindent\textbf{Putting Things Together.}
Let $F$ be a base model, $H_\xi$ be an embedding algorithm (e.g. $\OracleHCNN$).
Then the end-to-end model is $\Gamma(\bfx) = F(H_\xi(\bfx, F))$,
for some appropriately chosen $\xi$. We next show how to set parameters ($\xi, \delta$, etc.).
Consider an adversary who wants to perform an $\|\cdot\|$-attack with parameter $\eta$.
Suppose that we decide to use $\HCNN$ with parameter $\xi$.
Note that the guarantee that confidence being a discriminator
(Proposition~\ref{proposition:confident-separation}) holds in a neighborhood of radius $\delta$.
Thus to have the guarantee hold in the search of radius $\xi$, we must have $\delta \ge \xi + \eta$.

Suppose that $F$ satisfies $(p, q, \delta)$-separation. We say that $(\bfx, y) \sim \calD$ is
\emph{$(p, \delta)$-good} if $\{ (\forall y' \neq y, \bfx' \in N(\bfx, \delta)), F(\bfx')_{y'} < p \}$.
We note that $\Gamma_\xi^\MCN(\cdot) \equiv F(\MCN_\xi(\cdot))$ will output a correct prediction at a point
$\bfz \in N(\bfx, \delta)$, if $\bfz$ satisfies the following definition,
\begin{definition}[\textbf{$(p, \MCN_\xi)$-goodness}]
  A point $\bfz$ is said to be $(p, \MCN_\xi)$-good if there is a point in $N(\bfz, \xi)$
  that is $p$-confident for the correct label.
\end{definition}

We can further show that $\Gamma^\MCN_\xi$ improves separation:
\begin{proposition}[\textbf{$\MCN_\xi$ improves separation property}]
  \label{proposition:mcn-improves-separation}
  Suppose that $F$ satisfies $(p, q, \delta)$-separation. Let $\eta, \xi$ satisfy that
  $\eta + \xi \le \delta$. If for every $(p, \delta)$-good point $(\bfx, y)$, we have that every point
  $\bfz \in N(\bfx, \eta)$ is $(p, \MCN_\xi)$-good, then $\Gamma_\xi^\MCN$ satisfies $(1-p, q, \eta)$
  separation.
\end{proposition}
\begin{proof}
  By contraposition it suffices to prove the following
  \begin{align*}
    &\Pr_{(\bfx, y) \sim \calD}\big[
    (\forall y' \neq y, \bfx' \in N(\bfx, \eta)), \Gamma_\xi^\MCN(\bfx')_{y'} < 1-p
      \big] \\
    &\ge 1-q.
  \end{align*}
  By assumption that $F$ satisfies $(p, q, \delta)$-separation, with probability at least $1-q$ that
  $(\bfx, y) \sim \calD$, $(\bfx, y)$ is $(p, \delta)$-good. For every such $(p, \delta)$-good point
  $(\bfx, y)$, by assumption, every $\bfz \in N(\bfx, \eta)$ is $(p, \MCN_\xi)$-good.
  Therefore for every such $\bfz$, $\Gamma^\MCN_\xi(\bfz)_y \ge p$, and so
  $(\forall y' \neq y), \Gamma_\xi^\MCN(\bfz)_{y'} < 1-p$. The proof is complete.
\end{proof}
That is, $\Gamma_\xi^\MCN$ satisfies a much stronger separation property,
though in a smaller $\eta$-neighborhood, in the sense that there is no point in the $\eta$-neighborhood
at which $\Gamma_\xi^\MCN$ will assign confidence $\ge 1-p$ to a wrong label.

Finally, as it turns out in our experiments, there are nontrivial settings of $\delta, \eta, \xi$
under which there is a significant improvement of model robustness over the base model.


\section{Empirical Study}
\label{sec:experiments}
In this section we perform a detailed empirical study of our method.
We call a model ``natural'' if it is trained in the usual way without considering adversarial robustness,
and a model ``adversarially trained'' if it is trained using the paradigm specified in
(\ref{def:madry-objective}). A key objective of the empirical study is to compare the behavior of confidence
of adversarially trained models and that of natural models, as our analysis predicts that confidence will act
as a discriminator between right and wrong predictions for adversarially trained models.
We ask the following questions:
\begin{enumerate}
\item \emph{Does an adversarially trained model satisfy well our probabilistic
    separation property (Definition~\ref{def:p-q-delta-separation})?
  }
\item \emph{Is confidence information effective to help rejecting adversarial examples
    (e.g., by returning $\bot$), while retaining generalization on points sampled from the
    underlying distribution?}
\item \emph{Is $\HCNN_\xi$ effective in defending against adversarial perturbations,
    while retaining generalization capability of the adversarially trained base model?}
\end{enumerate}

\noindent\textbf{Overall Setup.}
We study the above questions using $\ell_\infty$ attacks over CIFAR10~\cite{K09}.
We denote by $\eta$ the radius parameter of $\ell_\infty$ attacks.
We reuse the ResNet model trained by~\cite{MMSTV17} as a representative of an adversarially trained model.
This model is trained with respect to an $\ell_\infty$-adversary with parameter
$\delta = 8/256 \approx 0.030$. We denote by $\xi$ the radius parameter of $\HCNN$.

\noindent\textbf{Attacks and Gradient Masking.}
We note that $\HCNN$ may induce an effect of \emph{gradient masking}~\cite{ACW18},
and thus may be susceptible to attacks that try to intentionally bypass gradient masking.
Taking this into account, we modified existing attacks, including the CW attack
in~\cite{CW17}, and the PGD attack~\cite{MMSTV17}, to \emph{exploit confidence information}.
Basically, these modified attacks find adversarial points \emph{with high confidence},
in order to break our defense method which relies on confidence as a discriminator.

\noindent\textbf{Summary of Findings.}
Our main findings are as follows: {\bf (1)} The behavior of confidence of adversarially trained models,
as predicted by our analysis, is significantly better than natural models in distinguishing right and
wrong predictions. In particular, within small radius, confidence information of Madry et al.'s model gives
a good discriminator to improve adversarial robustness. 
{\bf (2)} Unfortunately, confidence of the adversarially trained mode by~\cite{MMSTV17} loses control
as one approaches the $.03$ boundary. We give further reasoning and argue that
modifications to Madry et al.'s formulation seem necessary in order to defend against
$.03$-$\ell_\infty$ attacks effectively.

\begin{table}[bth]
  \centering
  \renewcommand{\arraystretch}{1.2}
  \begin{tabular}{c|c|c}
    \hline
    & Adv. trained model   & Natural model  \\ \hline
    \# adv. points  & 26  & 277  \\ \hline
  \end{tabular}
  \caption{Results for testing separation property. We randomly sample $300$ points.
    For each column, we report number of samples that we can find attacks of confidence at least $0.9$.
  }
  \label{table:goodness-results}
\end{table}

\begin{table*}[!bth]
  \centering
  \renewcommand{\arraystretch}{1.2}
  \begin{tabular}{c|c|c|c}
    & AWPR (original) & AWPR (w/ rejection) & Recall of $\calN$ \\ \hline
    $\eta=.010, p=.90$ & 11\% & 2.2\% & 86.7\% \\ \hline
    $\eta=.010, p=.95$ & 11\% & 1.1\% & 82.5\% \\ \hline
    $\eta=.010, p=.99$ & 11\% & 0.0\% & 72.6\% \\ \hline
    $\eta=.015, p=.90$ & 14\% & 5.4\% & 89.6\% \\ \hline
    $\eta=.015, p=.95$ & 14\% & 2.9\% & 85.6\% \\ \hline
    $\eta=.015, p=.99$ & 14\% & 0.3\% & 75.9\% \\ \hline
    $\eta=.020, p=.90$ & 21.7\% & 14.7\% & 90.0\% \\ \hline
    $\eta=.020, p=.95$ & 21.7\% & 9\% & 85.7\% \\ \hline
    $\eta=.020, p=.99$ & 21.7\% & 3.8\% & 75.4\% \\ \hline
    $\eta=.030, p=.90$ & 34.7\% & 29.7\% & 90.0\% \\ \hline
    $\eta=.030, p=.95$ & 34.7\% & 26.8\% & 86.0\% \\ \hline
    $\eta=.030, p=.99$ & 34.7\% & 16.0\% & 76.5\% \\ \hline
  \end{tabular}
  \caption{Results on rejecting adversarial examples with an adversarially trained model.
    We report adversarially-wrong-prediction-ratio (AWPR) before applying rejection rule,
    AWPR after applying rejection rule, and the recall of natural points.
  }
  \label{table:rejecting-adv}
\end{table*}

\begin{table*}[!bth]
  \centering
  \renewcommand{\arraystretch}{1.2}
  \begin{tabular}{c|c|c|c}
    & AWPR (original) & AWPR (w/ rejection) & Recall of $\calN$ \\ \hline
    $\eta=.010, p=.90$ & 99.3\% & 99.1\% & 97.2\% \\ \hline
    $\eta=.010, p=.95$ & 99.3\% & 99.1\% & 96.2\% \\ \hline
    $\eta=.010, p=.99$ & 99.3\% & 98.9\% & 92.6\% \\ \hline
  \end{tabular}
  \caption{Results on rejecting adversarial examples with a natural model.
    We only tested $\eta=.010$, the weakest setting of $\ell_\infty$ attack,
    and AWPRs are very high, even for $p=.99$. This indicates that the behavior
    of confidence of a natural model is poor.
  }
  \label{table:rejecting-adv-nat-model}
\end{table*}

\begin{table*}[bth]
  \centering
  \renewcommand{\arraystretch}{1.2}
  \begin{tabular}{c|c|c||c|c}
    &  \multicolumn{2}{|c||}{$\eta = \xi = 0.010$}  & \multicolumn{2}{c}{$\eta = \xi = 0.015$} \\ \cline{2-5}
    & Label change & Confidence reduction & Label change & Confidence reduction \\ \hline
    1st confident point & 94  & 8621 & 114 & 8511 \\ \hline
    2nd confident point & 208 & 27   & 260 & 20   \\ \hline
    Other points      & 50  & 0    & 66  & 0    \\ \hline
    Missing           & 0   & 0    & 29  & 0    \\ \hline\hline
    Total             & 352 & 8648 & 469 & 8531 \\\hline
  \end{tabular}
  \caption{$\MCN_\xi$ results. 9000 adversarial examples are divided into \emph{label change} (column 1 and 3)
    and \emph{confidence reduction} (column 2 and 4). Row 1 and 2 give the number of adversarial examples
    that $\MCN_\xi$ succeeded to narrow down to two most confident predictions over 10  perturbed points.
    Row 3 gives the number of cases where the correct label is not in the two most confident points,
    but is one of the model predictions over 10 perturbed points.
    Row 4 gives cases where all model predictions on $\MCN_\xi$ perturbed points are wrong.
  }
  \label{table:mcn-results}
\end{table*}

\noindent\textbf{Validating the Probabilistic Separation Property.}
To answer our first empirical question, we statistically estimate
$\Pr_{(x,y) \sim \calD}\big[(\exists y' \neq y, x' \in N(x, \delta)) F(x')_{y'} \ge p\big]$,
from Definition~\ref{def:p-q-delta-separation} and Proposition~\ref{proposition:confident-separation}.
We compare the \emph{adversarially trained model} in~\cite{MMSTV17} and its \emph{natural variant}.
Let $\calB$ denote the bad event $(\exists y' \neq y, x' \in N(x, \delta)) F(x')_{y'} \ge p$.
We randomly sample a batch of data points from the test set,
and compute frequency that $\calB$ happens for each batch.
To do so, we generate \emph{$p$-confident attacks} which are adversarial examples that have
model confidence at least $p$. We use a modified $\ell_\infty$-CW attack~\cite{CW17} to
search as long as possible to find $p$-confident adversarial examples that lie within the norm bound
$\eta = \delta = 0.030$. Table~\ref{table:goodness-results} summarizes the results.

With these statistics we can thus estimate $(p, q, \delta)$-separation
(Definition~\ref{def:p-q-delta-separation}) of the models in comparison.
We have (details are deferred to Appendix~\ref{sec:bound-estimation}).

\begin{proposition}[Separation from statistics]
  \label{prop:model-separatedness}
  Let $\calB$ be the event $\big\{ (\exists y' \neq y, \bfx' \in N(\bfx, \delta))
  F(\bfx')_{y'} \ge .9 \big\}$. The following two hold:
  \begin{itemize}
  \item With probability at least $.9$ the robust model in~\cite{MMSTV17} satisfies
    $\Pr_{(x,y) \sim \calD}\big[\calB\big] \le \frac{14}{75} = 0.18667\ldots$
    That is, the robust model has $(\frac{9}{10}, \frac{14}{75}, \frac{8}{255})$-separation.
  \item With probability at least $.9$ the natural model satisfies
    $\Pr_{(x,y) \sim \calD}\big[\calB\big] \ge \frac{247}{300} = 0.82333\ldots$
    That is, the natural model does \emph{not} have $(\frac{9}{10}, q, \frac{8}{255})$-separation for
    $q < \frac{247}{300}$.
  \end{itemize}
\end{proposition}

\noindent\textbf{Rejecting Adversarial Examples.}
To answer our second empirical question, we perform the following experiment.
We sample $1000$ points from the test set, where the base model is correct at, as a set of natural points $\calN$.
For each point $\bfx \in \calN$, we generate an adversarial example $\bfx^\star$
whose model confidence is maximal. To do so, we use a strengthened version of
the PGD attack used in~\cite{MMSTV17} (with $\ell_\infty$ radius $\eta$, $10$ random starts,
and $100$ iterations)\footnote{The PGD attack used in~\cite{MMSTV17} used $1$ random start
  and $20$ iterations on CIFAR.} to first generate, for each wrong label,
an adversarial example whose model confidence is as large as possible.
Then among all possible adversarial examples, we pick one that has the largest confidence.
We collect all the adversarial $\bfx^\star$ into $\calA$. Note that $|\calA| \le |\calN|$
because for each $\bfx \in \calN$ we can generate at most one $\bfx^\star$ (or we do not find any).
Now, for a confidence parameter $p$, we reject a point $\bfx$ as adversarial by ``$F(\bfx) < p$,''
where $F$ is the base ResNet model that is adversarially trained in~\cite{MMSTV17}.
A core metric we report is \emph{Adversarially Wrong Predictions Ratio} (AWPR),
the fraction of points where we can generate adversarially wrong predictions. Without the rejection rule,
the AWPR is $|\calA| / |\calN|$, and we hope that this ratio decreases significantly after applying
the rejection rule. Meanwhile, we also hope that the \emph{recall} of natural points is large enough
which means that we do not reject many natural points.

Table~\ref{table:rejecting-adv} and~\ref{table:rejecting-adv-nat-model} present the results
for adversarially and naturally trained models, respectively. In summary:
{\bf (1)} For adversarially trained models, rejection based on confidence gives a statistically significant
reduction in AWPR, while having a good recall of retaining natural points.
While the results are far from being perfect, they give encouraging evidence that our analysis
of Madry et al.'s formulation on confidence holds non-trivially in practice.
{\bf (2)} By contrast, rejection based on confidence has no effect at all with a natural model.
Even we set $p=0.99$, there are many adversarial examples that reach such confidence
using the weakest attack radius $.01$ in our experiments. This demonstrates that adversarial training
prepares confidence in a nontrivial way as a discriminator.
{\bf (3)} Finally, we notice that for the adversarially trained model we used,
the effectiveness of the rejection rule drops significantly as we approach the $.03$ boundary.
In particular, we are able to find many more high-confidence adversarial points there
than neighborhoods of smaller radius.

\noindent\textbf{Defending against adversarial examples using $\MCN_\xi$.}
We implemented a version of $\MCN_\xi$ and test its performance against $\ell_\infty$ attacks.
We sample a set $\calN$ of 1000 random test points where the base model produces \emph{correct} predictions.
We use PGD attacks of $1000$ iterations to increase the model confidence for target labels as much as possible.
With the attack, we first generate 9000 adversarial examples $\bfx^\star$
for each $\bfx \in \calN$ and each of the 9 incorrect target labels.
For a number of cases, the attack achieves different prediction $C_F(\bfx^\star)$ (\emph{label change}),
and for the other cases, the attack fails to change the prediction but reduces prediction confidence
of the correct label (\emph{confidence reduction}).

Our implementation of $\MCN_\xi$ solves (\ref{eq:mcnn-label}) using the PGD attack with a different setting
($\ell_\infty$ radius $\xi$, no random start, and 500 iterations).
Table~\ref{table:mcn-results} gives results where we report top-2 accuracy.
That is, let $\bfx_1, \ldots, \bfx_{10}$ be the 10 points perturbed by $\MCN_\xi$
for an input $\bfx$ with correct label $y$, sorted by the base model confidence
$\|F(\bfx_1)\|_\infty\ge,\ldots,\ge \|F(\bfx_{10})\|_\infty$.
We count the number of following two cases: {\bf (i)} $y = C_F(\bfx_1)$,
{\bf (ii)} $y \ne C_F(\bfx_1) \text{ but } y = C_F(\bfx_2)$. In summary:
{\bf (1)} $\MCN_\xi$ recovers a nontrivial number of correct predictions
from  adversarial examples with changed labels, without changing the originally correct predictions of
adversarial examples with reduced confidences. {\bf (2)} Specifically, for $\eta = \xi = 0.010$,
$\MCN_\xi$ produces correct predictions of 94 label-changed adversarial examples,
and wrong predictions for 27 confidence-reduced adversarial examples.
For $\eta = \xi = 0.015$, $\MCN_\xi$ produces correct predictions for 114 label-changed adversarial examples,
and wrong predictions for 20 confidence-reduced adversarial examples.
{\bf (3)} Interestingly, for most cases (99.44\% for $\eta = \xi = 0.010$, 98.94\%
for $\eta = \xi = 0.015$), the correct label $y$ can be predicted from the perturbed points
$\bfx_1, \bfx_2$ with two highest confidences.

\noindent\textbf{Potential Problems with Madry et al.'s Formulation.}
Our experiments demonstrate that as attack radius increases, the effectiveness of confidence decreases,
and it almost completely loses control at the $.03$ boundary. To this end, we first notice that
$\ell_\infty$ ball of radius $.03$ is very large: with an attack of radius $.03$ it is possible
to craft a good adversarial point, but one can also arrive at a almost random noise (by perturbing each
position of the input image by $.03$ for example). Given such a large ball to enforce robustness,
we observe two potential problems with Madry et al.'s formulation (\ref{def:madry-objective}):
{\bf (1)} \emph{gradient descent is not effective enough to search the space}.
We note that the inner maximum of (\ref{def:madry-objective}) is solved by gradient descent,
so probably one can only search a small neighborhood, and make confidence good there.
{\bf (2)} \emph{a more fundamental problem is that enforcing (\ref{def:madry-objective}) in a large
  $\ell_\infty$ ball is too much to hope for, because this ball contains random noise anyway.}
It seems more reasonable to us that one should only encourage (\ref{def:madry-objective})
in a relatively small neighborhood (say $.01$), which builds a \emph{trusted region},
and then encourage \emph{low-confidence} outside the neighborhood, so following confidence information
one can hope to go back to trusted regions.


\section{Related Work}
\label{sec:related}
\cite{SZSBEGF13} first observed the susceptibility of deep neural networks to adversarial perturbations.
Since then, a large body of work have been devoted to study hardening neural networks for this problem
(a small set of work in this direction is~\cite{GSS14, PMWJS16, MMKI17}).
Simultaneously, another line of work have been devoted to devise more effective or efficient attacks
(a small set of work in this direction is~\cite{MDFF16, PMJ+16, CW17}).
Unfortunately, there is still a large margin to defend against more sophisticated attacks,
such as Carlini-Wagner attacks~\cite{CW17}. For example, while the recent robust residual network constructed
by~\cite{MMSTV17} achieves encouraging robustness results on MNIST,
on CIFAR10 the accuracy against a strong adversary can be as low as $45.8\%$.

A line of recent research investigated using a generative model to ``explicitly recognizing''
natural manifolds to defend against adversarial perturbations.
Examples of this include the robust manifold approach by~\cite{IJADD17},
and an earlier proposal of MagNet by~\cite{MC17}. Similar to our work,
there is a discriminator that can distinguish between right and wrong predictions.
However, there is a crucial difference: In generative model approach, one tries to distinguish
between points on the manifolds and those that are not, and thus only \emph{indirectly} discriminate between
right and wrong predictions. By contrast, our confidence discriminator \emph{directly} discriminates
between right and wrong predictions. Therefore, one can think of our approach as
a ``discriminative'' alternative to the ``generative'' approach studied in robust manifold approach and MagNet,
and thus seems easier to achieve. In fact, MagNet has already been successfully attacked by~\cite{CW17-2}
by exploiting gaps between the generative discriminator and the underlying classifier.


\section{Conclusion}
\label{sec:conclusion}
In this work we take a first step to study structural information induced by an adversarial training to
further reinforce adversarial robustness. We show that confidence is a provable property, once one solves
the adversarial training objective of Madry et al. well, to distinguish between right and wrong predictions.
We empirically validate our analysis and report encouraging results. Perhaps more importantly, our analysis
and experiments also point to potential problems of Madry et al.'s training, which we hope,
may stimulate further research in adversarial training.


\section{Acknowledgments}
This work was partially supported by ARO under contract number W911NF-1-0405.

{\small
\bibliographystyle{natbib}
\bibliography{paper}
}
\onecolumn
\appendix
\section{Bounding the probability for $(p,q,\delta)$-separation}
\label{sec:bound-estimation}
This section gives details of our estimation of $(p,q,\delta)$-separation from statistics
in Table~\ref{table:goodness-results}. Note that event $\calE_b$ corresponds to a Bernoulli trial.
Let $X_1, \dots, X_t$ be independent indicator random variables, where
\begin{align*}
  X_i = \begin{cases}
    1 & \text{if $\calE_b$ happens, } \\
    0 & \text{otherwise}
  \end{cases},
\end{align*}
and $X = (\sum_{i=1}^tX_i)/t$. Recall Chebyshev's inequality:
\begin{theorem}[Chebyshev's Inequality]
  For independent random variables $X_1, \dots, X_t$ bounded in $[0, 1]$,
  and $X = (\sum_{i=1}^tX_i)/t$, we have $\Pr[|X - \Exp[X]| \ge \varepsilon] \le \frac{\Var[X]}{\varepsilon^2}$.
\end{theorem}
In our case, $\Exp[X] = \Exp[X_1] = \cdots = \Exp[X_t]$ and let it be $\mu$,
and let the computed frequency be $\hat{\mu}$ (observed value).
Thus $\Pr[|\hat \mu - \mu| \ge \varepsilon] \le 1/(4\varepsilon^2 t)$
since $\Var[X] = \mu (1-\mu)/t < 1/4t$.
We thus have the following proposition about $(p,q,\delta)$-separation.
\begin{proposition}
  Let $\alpha, \varepsilon \in [0,1]$. For sufficiently large $t$ where
  $\frac{1}{4\varepsilon^2 t} \le 1 - \alpha$ holds, we have:
  \begin{itemize}
  \item {\normalfont (Upper bound)} With probability at least $\alpha$,
    $\mu$ is smaller than $\hat \mu + \varepsilon$.
  \item {\normalfont (Lower bound)} With probability at least $\alpha$,
    $\mu$ is bigger than $\hat \mu - \varepsilon$.
  \end{itemize}
\end{proposition}
For example, we have guarantees for $\alpha = .9$ by putting $\varepsilon = .1$ and $t \ge 250$.


\end{document}